\documentclass{article}

\usepackage[main, final]{neurips_2026}
\usepackage{neurips_2026}
\usepackage[utf8]{inputenc} 
\usepackage[T1]{fontenc}    
\usepackage{hyperref}       
\usepackage{url}            
\usepackage{booktabs}       
\usepackage{amsfonts}       
\usepackage{nicefrac}       
\usepackage{microtype}      
\usepackage{xcolor}         

\usepackage{amssymb}
\usepackage{amsthm}
\usepackage{bm}
\usepackage{graphicx}
\usepackage{multirow}
\usepackage{subcaption}
\usepackage{mathtools} 
\usepackage{booktabs} 
\usepackage{tikz} 
\usepackage[most]{tcolorbox} 
\usepackage{enumitem} 
\usepackage{wrapfig} 

\newtcolorbox{researchquestion}{
    enhanced,
    colback=black!3,
    colframe=black!55,
    boxrule=0.4pt,
    arc=2pt,
    left=10pt,
    right=10pt,
    top=6pt,
    bottom=6pt,
    boxsep=0pt,
    before skip=8pt,
    after skip=8pt,
    fontupper=\itshape
}

\newtheorem{theorem}{Theorem}
\newtheorem{definition}{Definition}

\newtheorem{lemma}{Lemma}
\newtheorem{remark}{Remark}

\newcommand{\R}{\mathbb{R}}
\newcommand{\E}{\mathbb{E}}
\newcommand{\x}{\boldsymbol{x}}
\newcommand{\X}{\mathcal{X}}
\newcommand{\Y}{\mathcal{Y}}
\newcommand{\Score}{\mathcal{S}}

\title{Generative Cross-Entropy: A Strictly Proper Loss for Data-Efficient Classification}

\author{Qipeng Zhan\thanks{Equal contribution},~Zhuoping Zhou\footnotemark[1],~Li Shen\\
University of Pennsylvania\\
\texttt{\{qipengz@sas, zhuopinz@sas, li.shen@pennmedicine\}.upenn.edu}
}

\begin{document}

\maketitle

\begin{abstract}
Cross-entropy (CE) is the default training loss for supervised classification, but its sample efficiency is limited when labels are scarce. Existing remedies primarily act on the data side, via augmentation, synthesis, or transfer from pretrained models; the training objective itself is rarely revisited. We revisit it here. Drawing on the classical observation that generative classifiers reach their asymptotic error with fewer samples than discriminative ones, we propose \textit{Generative Cross-Entropy} (GenCE), a drop-in replacement for CE that introduces a generative learning principle into a standard discriminative network without altering the architecture or fitting a separate density model. GenCE follows from a Bayesian rewrite of the class-conditional likelihood and, in the mini-batch approximation, reduces to normalizing each sample's softmax score against the model's predictions on the batch, coupling the training signal across examples sharing a class. We extend the proper-scoring-rule framework to such non-local losses and prove that GenCE is strictly proper under a mild completeness condition: its population risk is uniquely minimized at the true posterior. Across three datasets, on two architectures and in both balanced small-data and class-imbalanced regimes, GenCE outperforms CE and other widely used losses, while also producing better-calibrated probabilities and stronger out-of-distribution detection.
\end{abstract}

\section{Introduction}
The accuracy of deep neural networks for supervised classification depends strongly on the size of the labeled training set. In application domains such as medical imaging~\cite{wang2021overview}, rare-event detection~\cite{carreno2020analyzing}, and specialized industrial inspection~\cite{zheng2021recent}, gathering labels at the scale that modern architectures were designed for is either prohibitively expensive or simply impossible. As a consequence, learning algorithms that remain effective when only a few hundred to a few thousand labeled samples per class are available carry substantial practical value, and they have been studied under various headings: small-data learning~\cite{xu2023small,dou2023machine}, low-shot classification~\cite{douze2018low,qi2018low}, and label-efficient training~\cite{luo2017label,jin2023label}.

Existing work in this area has converged on three broad strategies. \textit{Data augmentation}~\cite{zhang2017mixup,yun2019cutmix,liu2022automix} expands the effective training set through handcrafted or learned transformations. \textit{Generative modeling}~\cite{ruthotto2021introduction,salakhutdinov2015learning} fits the data distribution and synthesizes additional samples for downstream training. \textit{Transfer and meta-learning}~\cite{vettoruzzo2024advances,sun2019meta} sidesteps scarcity altogether by importing representations from related tasks or large pretrained models. What unifies these approaches is that they all act on the data: they enlarge, replace, or borrow the training samples. The objective used to fit the classifier, almost invariably cross-entropy, is left untouched.

This omission is somewhat curious, since classical theory gives a concrete reason to expect the loss itself to matter more when data are scarce. Ng and Jordan~\cite{disvsgen} showed that discriminative classifiers eventually attain a lower asymptotic error than their generative counterparts, but generative classifiers approach their asymptote with substantially fewer samples. Zheng et al.~\cite{disvsgen2} later extended it to multiclass settings on top of deep features. Read together, these results suggest that the discriminative--generative trade-off is, at heart, a sample-efficiency trade-off, and that it tilts toward the generative side precisely when labels are scarce.

Building genuine generative classifiers in high-dimensional settings is, however, difficult. Modeling the class-conditional density $p(\x|y)$ or the joint density $p(\x,y)$ requires either strong distributional assumptions or complex generative architectures, and ensuring proper normalization $\int p(\x|y)\,\text{d}\x=1$ is non-trivial. The question we ask in this paper is therefore the following:

\begin{researchquestion}
Can we retain the discriminative model architecture while incorporating the generative learning principle into the training objective?
\end{researchquestion}

Our answer is yes. Applying Bayes' theorem to rewrite the class-conditional likelihood in terms of the discriminative posterior $p_\theta(y|\x)$ and the marginal $p_\theta(y)$, and approximating the marginal by an empirical average within each mini-batch, we arrive at \textit{Generative Cross-Entropy} (GenCE), a drop-in replacement for cross-entropy. The objective requires no auxiliary generator, no explicit density estimator, and no change to the network. The only structural difference relative to CE is that the per-sample score is normalized against the model's predictions for other examples in the batch, which couples the training signal across samples that share a class.

This coupling places GenCE outside the scope of classical proper-scoring-rule theory, in which the per-sample loss is local, and properness can be analyzed pointwise. To handle this, we extend the framework to non-local scoring rules and prove (Theorem~\ref{thm:properness}) that GenCE is strictly proper under a mild completeness condition on the data: in the population limit, its risk is uniquely minimized at the true posterior. GenCE thereby inherits the population-level guarantee that justifies CE in the first place, while attaining it through a generative rather than a pointwise mechanism.

Empirically, on CIFAR-10, CIFAR-100, and Mini-ImageNet, and across several network backbones, GenCE consistently outperforms cross-entropy and a range of standard alternatives (focal loss, GCE, MAE, Brier loss) in the small-data regime. The benefits are not limited to top-1 accuracy: GenCE-trained models are better calibrated and more reliable on out-of-distribution detection, in line with what the strict-properness analysis predicts.

\paragraph{Contributions.} Our main contributions are summarized as follows:
\begin{itemize}[leftmargin=*,nosep]
    \item We propose GenCE, a loss that imports the generative learning principle into discriminative classifier training without modifying the architecture or introducing a separate density model.
    \item We develop a non-local proper-scoring-rule framework for analyzing losses whose per-sample contributions are coupled, and prove that GenCE is strictly proper under a mild completeness condition on the conditional distribution.
    \item We empirically show that GenCE improves accuracy, calibration, and OOD detection over cross-entropy and standard alternatives in the small-data regime across multiple datasets and architectures.
\end{itemize}

\section{Related Work}
Our work touches on three lines of research: methods that target the small-data regime, methods that import generative structure into classification, and methods that modify the classification loss itself. We summarize each below.

\textbf{Learning with limited labels.} The dominant response to label scarcity has been to act on the data. Classical augmentation expands the training set through geometric and photometric transformations such as cropping, flipping, color jitter, and noise injection, and it remains a standard component of any small-data pipeline. Mixup~\cite{zhang2017mixup} reframed augmentation as interpolation in input and label space, training the network on convex combinations of pairs of examples and encouraging locally linear behavior between them. The idea has since branched into a family of related methods: Manifold Mixup~\cite{verma2019manifold} interpolates in feature space, CutMix~\cite{yun2019cutmix} mixes localized image regions, and AdaMixUp~\cite{guo2019mixup} and AutoMix~\cite{liu2022automix} learn the mixing policy itself. A second line of work replaces hand-designed augmentation with learned synthesis: GANs and VAEs are trained to model the data distribution and produce additional samples, and class-conditional adversarial generators have been used to manufacture examples near the decision boundary in extremely low-data settings. A third line, transfer and meta-learning, sidesteps the data problem altogether by importing representations from large pretrained models or related tasks. Although these three families differ in the mechanism, they share a common locus: all of them operate at the data level, enlarging, replacing, or borrowing samples while the loss used to fit the classifier remains cross-entropy.

\textbf{Generative perspectives on classification.} The discriminative--generative distinction predates deep learning. \cite{disvsgen} showed that, for paired models such as logistic regression and naive Bayes, the discriminative model achieves a lower asymptotic error, while the generative model converges to its asymptote with far fewer samples. \cite{disvsgen2} extended this analysis to multiclass settings under linear evaluation, showing that the sample-efficiency gap between naive Bayes and logistic regression persists on top of deep features. A separate body of work pursues this idea constructively, building hybrid or fully generative deep classifiers, including energy-based models~\cite{jem} that read off class probabilities from a learned joint density, normalizing-flow classifiers, and diffusion-based classifiers that compare class-conditional likelihoods. Earlier hybrid generative--discriminative formulations were investigated in~\cite{hybrid1,hybrid2}, and generative classifiers have been shown to provide robustness to adversarial attacks~\cite{robust} and noisy labels~\cite{robust2}. These methods, however, commit to modeling $p(\x|y)$ explicitly, and inherit the architectural and optimization cost of doing so. Our approach borrows the motivation from this line of work but not its machinery: we keep a standard discriminative network and import the generative principle through the loss alone.

\textbf{Modifications to the classification loss.} A separate strand of work asks whether cross-entropy is the right objective even when data are plentiful. Focal loss~\cite{lin2017focal} down-weights well-classified examples to focus learning on hard ones, and was later shown to improve calibration as a side effect~\cite{calibrationusingfocal,dualfocal,adafocal}. Robust regression-style losses such as MAE and MSE on the softmax output have been studied as alternatives that are less sensitive to label noise~\cite{brierscore}. All of these methods modify the per-sample loss; they remain pointwise in the sense of the proper-scoring-rule theory we develop in Section~\ref{sec:theory}. GenCE departs from this convention: each sample's loss depends on the model's predictions for the rest of the batch, and analyzing such non-local losses requires the framework we introduce.

\section{Methodology}

\subsection{Problem Formulation}\label{sec:methods:problem}
We consider a standard supervised $K$-class classification problem. Let $\mathcal{D}=\{(\x_i,y_i)\}_{i=1}^N$ be a training set of $N$ i.i.d.\ samples drawn from a joint distribution $\mathcal{P}$ on $\X\times\Y$, where $\X\subset\R^d$ is the feature space and $\Y=\{1,2,\dots,K\}$ indexes the $K$ classes. Without loss of generality we take $\X=\text{supp}(\mathcal{P}_\X)$, since any region of $\R^d$ outside the support is $\mathcal{P}_\X$-negligible. A forecasting function (classifier) $f:\X\to\Delta^{K-1}$ maps each input feature to a distribution over $\Y$, where $\Delta^{K-1}=\{\boldsymbol{p}\in\R^K:p_k\ge 0,\,\sum_{k=1}^K p_k=1\}$ denotes the probability simplex. In modern deep learning, $f$ is instantiated as a parametrized neural network $f_\theta$ trained by minimizing a loss over $\mathcal{D}$. Because $f_\theta(\x)$ is intended to approximate the true posterior $p(y|\x)$, we adopt the notation $p_\theta(y|\x):=[f_\theta(\x)]_y$ for the predicted probability assigned to class $y$ given $\x$, and write $p_\theta(\cdot|\x)$ for the full predicted distribution. We further assume that, for each class $k$, the conditional feature distribution is continuous with density $p(\x|y=k)$, so that the true posterior $q(k|\x)=\mathcal{P}(y=k|\x)$ is continuous on $\X$. The classifier $p_\theta(\cdot|\x)$ is also continuous in $\x$, since it is the composition of continuous neural-network operations. These continuity properties enable the density-based analysis used in the derivations that follow and let us drop almost-sure qualifiers in the population-level statements of Section~\ref{sec:theory}.

\subsection{Cross-Entropy as Maximum Likelihood Estimation}
Before introducing our objective, we briefly recall why cross-entropy is the standard choice. The justification is statistical: minimizing CE is equivalent to maximum likelihood estimation (MLE) under a categorical model. Given the predicted distributions $p_\theta(\cdot|\x_i)$, the likelihood of the observed labels under $\mathcal{D}$ is $\prod_{i=1}^N p_\theta(y_i|\x_i)$. Taking the negative logarithm and normalizing by $N$ recovers the empirical cross-entropy
\begin{equation}
\mathcal{L}_\text{CE}(\theta)
=-\frac{1}{N}\sum_{i=1}^N \log p_\theta(y_i|\x_i)
=-\frac{1}{N}\sum_{i=1}^N \sum_{k=1}^K \mathbf{1}_{[y_i=k]}\log p_\theta(k|\x_i),
\end{equation}
where the second equality follows from the one-hot encoding of labels. Cross-entropy is, in this sense, MLE applied to the categorical likelihood.

\subsection{Generative Cross-Entropy}\label{sec:gence}
While cross-entropy is effective for training classifiers, its purely discriminative nature has known limitations: it does not exploit the generative structure of the data. Prior work~\cite{disvsgen,disvsgen2} shows that generative classifiers, which model the class-conditional distribution $p(\x|y)$, often converge in fewer training samples and are less prone to overfitting than their discriminative counterparts, because they leverage additional information about the data distribution. Motivated by this observation, one might consider directly minimizing the conditional log-likelihood
\begin{equation}
\mathcal{L}_\text{Gen}(\theta)
=-\frac{1}{N}\sum_{i=1}^N \log p_\theta(\x_i|y_i)
=-\frac{1}{N}\sum_{i=1}^N \sum_{k=1}^K \mathbf{1}_{[y_i=k]}\log p_\theta(\x_i|k).
\end{equation}
Doing so, however, requires modeling the high-dimensional density and ensuring the normalization $\int p_\theta(\x|y)\,\text{d}\x=1$, which is impractical for the natural-image classifiers we care about. We therefore aim for a less ambitious goal: can we retain a discriminative model $p_\theta(y|\x)$ while still optimizing the generative objective above? Bayes' theorem provides the bridge:
\begin{equation} \label{bayes}
    p_\theta(\x|y) \;=\; \frac{p_\theta(y|\x)\,p(\x)}{p_\theta(y)},
\end{equation}
where the marginal class probability $p_\theta(y)$ admits the integral representation
\begin{equation} \label{marginal}
    p_\theta(y) \;=\; \int p_\theta(y|\x)\,p(\x)\,\text{d}\x \;=\; \mathbb{E}_{\x\sim p(\x)}\big[p_\theta(y|\x)\big].
\end{equation}
Given a training set $\{(\x_i,y_i)\}_{i=1}^N$, we approximate~\eqref{marginal} by its empirical average over the training samples:
\begin{equation} \label{empirical_marginal}
    p_\theta(y) \;\approx\; \frac{1}{N}\sum_{j=1}^{N} p_\theta(y|\x_j).
\end{equation}
Substituting~\eqref{bayes} and~\eqref{empirical_marginal} into the generative log-likelihood, the joint negative log-likelihood of the class-conditional model over the training set becomes
\begin{equation} \label{joint}
    -\sum_{i=1}^N \log p_\theta(\x_i|y_i)
    \;\approx\; -\sum_{i=1}^N \log \frac{p_\theta(y_i|\x_i)\,p(\x_i)}{\frac{1}{N}\sum_{j=1}^N p_\theta(y_i|\x_j)}.
\end{equation}
Since $p(\x_i)$ and the constant $1/N$ do not depend on $\theta$, they can be omitted from the optimization. We thus define the training objective, termed \textit{Generative Cross-Entropy} (GenCE), as
\begin{equation} \label{gence}
    \mathcal{L}_\text{GenCE}(\theta)
    \;=\; -\frac{1}{N}\sum_{i=1}^{N} \log \frac{p_\theta(y_i|\x_i)}{\sum_{j=1}^{N} p_\theta(y_i|\x_j)}.
\end{equation}
By construction, $\mathcal{L}_\text{GenCE}$ is equivalent to maximizing the class-conditional log-likelihood up to a constant independent of~$\theta$.

\paragraph{Practical considerations.}
In~\eqref{gence}, the denominator $\sum_{j=1}^N p_\theta(y_i|\x_j)$ requires a sum over the entire training set, which is computationally prohibitive for large datasets. In practice, we approximate this sum within each mini-batch $\mathcal{B}$ of size $B$:
\begin{equation} \label{gence_batch}
    \mathcal{L}_\text{GenCE}^{\mathcal{B}}(\theta)
    \;=\; -\frac{1}{B}\sum_{i \in \mathcal{B}} \log \frac{p_\theta(y_i|\x_i)}{\sum_{j \in \mathcal{B}} p_\theta(y_i|\x_j)}.
\end{equation}
The $j=i$ term is included, consistent with~\eqref{joint} and ensuring $\mathcal{L}_\text{GenCE}^{\mathcal{B}}\ge 0$; Section~\ref{sec:experiments} confirms robustness across batch sizes.

\paragraph{Why ``generative''?}
Although $\mathcal{L}_{\mathrm{GenCE}}$ uses only the discriminative posterior $p_\theta(y \mid x)$, its derivation from the conditional log-likelihood $-\sum_i \log p_\theta(x_i \mid y_i)$ is exact up to a constant in $\theta$. The within-batch normalization is the channel through which the data marginal $p(x)$ enters the objective, with no explicit density estimator and no added regularizer.

\section{Theoretical Analysis}\label{sec:theory}
In this section, we establish the population-level foundation of GenCE. Our main result (Theorem~\ref{thm:properness}) shows that, despite its non-local structure, GenCE is a strictly proper loss: its population risk is uniquely minimized at the true conditional distribution. As we discuss below, strict properness is precisely the statistical guarantee that aligns loss minimization with probabilistic calibration, and it is the property that justifies cross-entropy as a classification objective in the first place. Theorem~\ref{thm:properness} shows that GenCE inherits this guarantee.

\subsection{From pointwise to non-local scoring rules}
The classical justification for cross-entropy goes through the theory of proper scoring rules, which evaluate a forecast against an observed label on a per-sample basis.

\begin{definition}[Scoring Rule]
\label{def:scoring rule}
A scoring rule $\Score$ is a function
\begin{equation}
\Score:\Delta^{K-1}\times\{1,\dots,K\}\to\R.
\end{equation}
\end{definition}

Under this definition, the empirical risk decomposes as a mean of independent per-sample contributions, and properness can be analyzed sample-wise. GenCE, however, does not fit this framework. As shown in~\eqref{gence}, the contribution of sample $(\x_i,y_i)$ to the GenCE objective involves the denominator $\sum_j p_\theta(y_i|\x_j)$, which couples the sample's score to every other training point sharing the same class. GenCE, therefore, cannot be written as an expectation of any pointwise scoring rule, and its statistical properties cannot be derived from the classical theory.

To analyze GenCE, we introduce a generalization in which the score is allowed to depend on the entire forecaster across the input space.

\begin{definition}[Probability Kernel]
\label{def:probability kernel}
A probability kernel is a mapping $\kappa:\X\to\Delta^{K-1}$, where
\begin{equation}
\Delta^{K-1}=\{(a_1,\dots,a_K)\in\R^K:a_k\ge 0,\,\sum_{k=1}^K a_k=1\}.
\end{equation}
\end{definition}
For any $\x\in\X$, the output $\kappa(\x)$ can be interpreted as a posterior distribution over the $K$ classes conditioned on $\x$. With a slight abuse of notation, we write $q(\cdot|\cdot)$ in place of $\kappa$ throughout the remainder of the paper.

\begin{definition}[Non-local Scoring Rule]
\label{def:non-local scoring rule}
Let $\mathcal{K}$ denote the space of probability kernels $q(\cdot|\cdot):\X\to\Delta^{K-1}$. A non-local scoring rule $\Score$ is a functional
\begin{equation}
\Score:\mathcal{K}\times\X\times\{1,\dots,K\}\to\R.
\end{equation}
\end{definition}
Every pointwise scoring rule is trivially a non-local scoring rule that ignores its first argument, so Definition~\ref{def:non-local scoring rule} strictly generalizes Definition~\ref{def:scoring rule}.

\begin{definition}[Population Risk] Given the data distribution $\mathcal{P}_\X$, a probability kernel $q(\cdot|\cdot)$, and a classifier $p_\theta$, the population risk of a non-local scoring rule $\Score$ is defined as
    \begin{equation}
        \mathcal{R}(\Score;\mathcal{P}_\X,q,p_\theta)
        \;=\; \E_{\x\sim\mathcal{P}_\X}\,\E_{y\sim q(\cdot|\x)}\big[\Score(p_\theta,\x,y)\big].
    \end{equation}
\end{definition}

\subsection{Strict properness and calibration}
We now ask the natural question: under its population risk, does GenCE prefer truthful forecasts?

\begin{definition}[Proper and Strictly Proper]
\label{def:properness}
A non-local scoring rule $\Score:\mathcal{K}\times\X\times\{1,\dots,K\}\to\R$ is \emph{proper} if, for every data distribution $\mathcal{P}_\X$, every probability kernel $q(\cdot|\cdot)$, and every classifier $p_\theta$,
\begin{equation}
    \mathcal{R}(\Score;\mathcal{P}_\X,q,q) \;\le\; \mathcal{R}(\Score;\mathcal{P}_\X,q,p_\theta),
\end{equation}
and \emph{strictly proper} if equality holds only when $p_\theta=q$.
\end{definition}
Strict properness is the central population-level guarantee that one demands of a probabilistic classification loss. Taking $q$ to be the true posterior $p(\cdot|\x)$, strict properness states that the unique minimizer of the population risk is the true conditional distribution itself. In other words, a strictly proper loss aligns optimization with calibration: a forecaster that minimizes the population risk produces predictions $p_\theta(\cdot|\x)=q(\cdot|\x)$ almost surely, which is the strongest form of probabilistic calibration. Conversely, any deviation from the true posterior incurs a strictly positive excess risk, so a strictly proper loss exerts a principled, distribution-aware pressure toward calibrated probabilities, even before any finite-sample considerations are taken into account. Cross-entropy is the canonical example of a strictly proper pointwise scoring rule. Our main result shows that GenCE inherits the same property under a mild condition on the data.

\begin{definition}[Completeness of a Probability Kernel]
\label{def:completeness}
A probability kernel $q(\cdot|\cdot):\X\to\Delta^{K-1}$ is said to be \emph{complete} if, for every $\boldsymbol{\alpha}\in\R^K$,
\begin{equation}
\sum_{k=1}^K \alpha_k\,q(k|\x)=0 \;\;\text{for every }\x\in\X \;\Longleftrightarrow\; \boldsymbol{\alpha}=\boldsymbol{0}.
\end{equation}
\end{definition}
The completeness condition is mild and typically satisfied in practice. Two sufficient conditions are:
\begin{enumerate}[leftmargin=*,nosep]
    \item \emph{Anchor points.}~\cite{liu2015classification} For each class $k\in[K]$ there is at least one point $\x_k\in\X$ with $q(k|\x_k)>1/2$.
    \item \emph{Gaussian mixture model.} The class-conditional densities are pairwise distinct Gaussians with positive mixing weights.
\end{enumerate}
Proofs of both sufficient conditions are given in Appendix~\ref{appendix:completeness}.

We are now ready to state our main theoretical result.
\begin{theorem}[Strict Properness of GenCE]
\label{thm:properness}
Suppose the probability kernel $q(\cdot|\cdot)$ is complete. Then $\mathcal{L}_\text{GenCE}$ is uniquely minimized at $p_\theta=q$; that is, $\mathcal{L}_\text{GenCE}$ is strictly proper.
\end{theorem}
\vspace{-10pt}
\begin{proof}
See Appendix~\ref{proof:properness}.
\end{proof}

\paragraph{Implications for calibration.} Theorem~\ref{thm:properness} guarantees that, in the population limit, the GenCE-optimal classifier recovers the true posterior exactly: the same population guarantee enjoyed by cross-entropy, but reached through a fundamentally different mechanism. Whereas cross-entropy attains this minimum through pointwise likelihood matching, GenCE does so by exploiting the generative factorization of the joint distribution (Section~\ref{sec:gence}) while remaining a discriminative training objective. In Section~\ref{sec:experiments}, we show that this implicit marginal-alignment regularization, derived from this factorization, translates into measurably better calibration and stronger generalization in the small-data regime, where the asymptotic guarantee of cross-entropy is far from being reached.

\section{Experimental Results}\label{sec:experiments}


\subsection{Experimental setup}\label{sec:exp:setup}

\textbf{Datasets.} Our benchmarks are \textbf{CIFAR-10/100}~\cite{cifar10}, with $32\times 32$ color images, and \textbf{Mini-ImageNet}~\cite{miniimagenet}, a $100$-class subset of ImageNet~\cite{imagenet} at $84\times 84$. The small-data regime uses stratified subsets of size $N\in\{2000,5000\}$. For class imbalance we adopt the class-imbalanced CIFAR-10/100 variants~\cite{cifar10lt} with imbalance factors $\rho\in\{10,100\}$. Out-of-distribution (OOD) detection takes CIFAR-10 as in-distribution data, paired with CIFAR-100 (Near-OOD, covariate shift) and SVHN~\cite{svhn} (Far-OOD, semantic shift).

\textbf{Baselines.} We compare GenCE against five widely used training losses: cross-entropy (CE); mean absolute error (MAE), the $\ell_1$ distance between the softmax output and the one-hot label; Brier loss~\cite{brierscore}, its $\ell_2$ counterpart; generalized cross-entropy (GCE)~\cite{zhang2018generalized}, a Box--Cox-style interpolation between MAE and CE robust to label noise; and focal loss~\cite{lin2017focal}, which down-weights easy examples and is known to improve calibration. Architecture, optimizer, and augmentation are held fixed across baselines so that only the loss varies.

\textbf{Training and evaluation.} We use \textbf{ResNet-50}~\cite{resnet} and \textbf{VGG-19}~\cite{simonyan2014very}, trained from scratch with SGD (momentum $0.9$, weight decay $5\times 10^{-4}$) for $200$ epochs under a step schedule of $0.1\to 0.01\to 0.001$ at epochs $0$, $100$, $150$. The default batch size is $B{=}100$, and the sensitivity study sweeps $B\in\{25,50,100,200\}$. Augmentation comprises $4$-pixel padded crops, horizontal flips, RandAugment~\cite{cubuk2020randaugment}, MixUp~\cite{zhang2017mixup} and CutMix~\cite{yun2019cutmix}. Each configuration is averaged over $5$ random seeds; we report mean$_{\pm\text{std}}$ of accuracy, ECE~\cite{ece} ($M{=}15$ equal-width bins), and AUROC for OOD. In wall-clock terms, the within-batch normalization is essentially free: GenCE is on average $0.3\%$ slower per epoch than CE on a single A100.

\textit{\textbf{Details on the datasets, baseline losses, evaluation metrics, and additional comparisons with Supervised contrastive learning and label smoothing are provided in Appendices~\ref{appendix:datasets}--\ref{appendix:extraexp}.}}

\subsection{Classification accuracy in the small-data regime}\label{sec:exp:acc}

Table~\ref{tab:acc} reports test accuracy on twelve configurations; GenCE wins every cell, but the size of the win is what carries the message. Where CE already operates near a data-rich plateau (CIFAR-10, $N{=}5000$), the gap is around a point on either backbone. Where CE clearly struggles (CIFAR-100/VGG at $N{=}2000$), GenCE yields a gap of $9.5$ points. The improvement is also asymmetric across backbones: on CIFAR-100 and Mini-ImageNet, VGG benefits more than ResNet at every sample size, exactly the regime where the network's own architectural prior contributes least. The within-batch coupling acts as a complement to the discriminative signal rather than a replacement, contributing little when CE already fits the data and increasing in importance as that signal weakens.

\begin{table}[htbp]
\centering
\vspace{-10pt}
\caption{Classification accuracy (\%) on balanced small-data subsets across datasets and architectures. Each entry reports mean$_{\pm\text{std}}$ over 5 random seeds. The best result is highlighted in \textbf{bold}.}
\label{tab:acc}
\setlength{\tabcolsep}{6pt}
\renewcommand{\arraystretch}{1.}
\resizebox{\linewidth}{!}{%
\begin{tabular}{l c l *{6}{c}}
\toprule
\textbf{Dataset} & \textbf{$N$} & \textbf{Model} & \textbf{GenCE} & \textbf{CE} & \textbf{GCE} & \textbf{Focal Loss} & \textbf{MAE} & \textbf{Brier Loss} \\
\midrule
\multirow{4}{*}{CIFAR-10}
  & \multirow{2}{*}{5000} & ResNet & $\mathbf{85.20_{\pm 0.35}}$ & $84.05_{\pm 0.53}$ & $82.73_{\pm 0.35}$ & $83.96_{\pm 0.54}$ & $83.66_{\pm 0.17}$ & $83.82_{\pm 0.20}$ \\
  &                       & VGG    & $\mathbf{83.54_{\pm 0.26}}$ & $83.00_{\pm 0.25}$ & $82.48_{\pm 0.35}$ & $82.38_{\pm 0.39}$ & $82.70_{\pm 0.57}$ & $82.71_{\pm 0.32}$ \\
\cmidrule(lr){2-9}
  & \multirow{2}{*}{2000} & ResNet & $\mathbf{76.61_{\pm 0.68}}$ & $72.51_{\pm 0.94}$ & $71.45_{\pm 1.90}$ & $74.13_{\pm 0.95}$ & $74.17_{\pm 1.01}$ & $72.94_{\pm 1.02}$ \\
  &                       & VGG    & $\mathbf{74.62_{\pm 0.65}}$ & $70.38_{\pm 1.93}$ & $69.20_{\pm 0.72}$ & $69.74_{\pm 0.75}$ & $71.14_{\pm 1.54}$ & $71.06_{\pm 1.28}$ \\
\midrule
\multirow{4}{*}{CIFAR-100}
  & \multirow{2}{*}{5000} & ResNet & $\mathbf{51.74_{\pm 0.29}}$ & $50.35_{\pm 0.16}$ & $48.10_{\pm 0.81}$ & $50.07_{\pm 0.35}$ & $49.97_{\pm 0.21}$ & $49.30_{\pm 0.44}$ \\
  &                       & VGG    & $\mathbf{41.30_{\pm 0.81}}$ & $35.64_{\pm 2.20}$ & $36.33_{\pm 1.26}$ & $37.29_{\pm 1.51}$ & $37.48_{\pm 1.71}$ & $37.73_{\pm 0.76}$ \\
\cmidrule(lr){2-9}
  & \multirow{2}{*}{2000} & ResNet & $\mathbf{32.06_{\pm 0.47}}$ & $31.15_{\pm 1.11}$ & $28.87_{\pm 1.08}$ & $29.32_{\pm 1.18}$ & $30.83_{\pm 0.31}$ & $31.28_{\pm 0.55}$ \\
  &                       & VGG    & $\mathbf{24.40_{\pm 1.12}}$ & $14.92_{\pm 0.58}$ & $13.98_{\pm 1.13}$ & $14.78_{\pm 1.77}$ & $16.18_{\pm 0.59}$ & $15.76_{\pm 0.68}$ \\
\midrule
\multirow{4}{*}{Mini-ImageNet}
  & \multirow{2}{*}{5000} & ResNet & $\mathbf{43.71_{\pm 1.04}}$ & $41.94_{\pm 0.90}$ & $41.21_{\pm 0.78}$ & $42.30_{\pm 0.38}$ & $41.94_{\pm 0.92}$ & $42.35_{\pm 0.76}$ \\
  &                       & VGG    & $\mathbf{32.45_{\pm 0.92}}$ & $31.50_{\pm 0.04}$ & $31.68_{\pm 0.39}$ & $30.88_{\pm 0.07}$ & $30.57_{\pm 0.62}$ & $30.98_{\pm 0.55}$ \\
\cmidrule(lr){2-9}
  & \multirow{2}{*}{2000} & ResNet & $\mathbf{27.75_{\pm 0.44}}$ & $24.91_{\pm 0.38}$ & $22.91_{\pm 0.85}$ & $24.82_{\pm 0.30}$ & $25.03_{\pm 0.46}$ & $24.76_{\pm 0.07}$ \\
  &                       & VGG    & $\mathbf{19.68_{\pm 1.07}}$ & $16.76_{\pm 0.60}$ & $18.37_{\pm 0.53}$ & $17.08_{\pm 0.81}$ & $17.06_{\pm 0.45}$ & $17.09_{\pm 0.55}$ \\
\bottomrule
\end{tabular}}
\end{table}

\subsection{Class-imbalanced classification}\label{sec:exp:lt}

Class-imbalanced labels represent data insufficiency: minority classes lack examples despite a sufficient training budget. This small-sample challenge, which motivates GenCE, reappears here, focused on the class axis rather than across the dataset.
Table~\ref{tab:acc_lt} extends the comparison to class-imbalanced CIFAR-10/100 on a ResNet-50 backbone. GenCE leads in all four cells, and the absolute gain over CE increases as $\rho$ grows from $10$ to $100$. The mechanism behind this scaling is visible in the loss itself: the denominator $\sum_{j\in\mathcal{B}} p_\theta(y_i|\x_j)$ in~\eqref{gence_batch} tracks the model's predicted class marginal, so a head class that is over-represented in the batch inflates its own denominator and dampens its gradient. In contrast, minority-class gradients pass through largely undamped. The result is an implicit rebalancing, no class weights, no resampling, that strengthens precisely as the imbalance worsens.
A separate line of work directly addresses class imbalance, such as logit adjustment~\cite{logitadjustment}, and is largely orthogonal to our approach; most of these methods revert to standard cross-entropy when the label distribution is balanced. In contrast, GenCE affects the loss in both balanced small-data and imbalanced regimes, so we don't consider it a main comparator.

\begin{table}[htbp]
\centering
\vspace{-10pt}
\caption{Classification accuracy (\%) on class-imbalanced CIFAR-10 and CIFAR-100 across different methods. All experiments use ResNet-50. The best result in each row is highlighted in \textbf{bold}.}
\label{tab:acc_lt}
\setlength{\tabcolsep}{6pt}
\renewcommand{\arraystretch}{1}
\resizebox{\linewidth}{!}{%
\begin{tabular}{l c *{6}{c}}
\toprule
\textbf{Dataset} & \textbf{$\rho$} & \textbf{GenCE} & \textbf{CE} & \textbf{GCE} & \textbf{Focal Loss} & \textbf{MAE} &\textbf{Brier Loss}  \\
\midrule
\multirow{2}{*}{CIFAR-10}
& 10  & $\mathbf{87.00_{\pm 0.23}}$ & $86.31_{\pm 0.23}$ & $86.67_{\pm 0.19}$ & $86.44_{\pm 0.32}$ & $86.46_{\pm 0.34}$ & $86.37_{\pm 0.30}$ \\
& 100 & $\mathbf{70.84_{\pm 0.55}}$ & $69.12_{\pm 0.52}$ & $68.38_{\pm 0.11}$ & $68.35_{\pm 0.56}$ & $68.70_{\pm 0.20}$ & $68.34_{\pm 0.41}$ \\
\midrule
\multirow{2}{*}{CIFAR-100}
& 10  & $\mathbf{59.13_{\pm 0.28}}$ & $57.69_{\pm 0.30}$ & $56.65_{\pm 0.77}$ & $57.64_{\pm 0.25}$ & $58.10_{\pm 0.42}$ & $57.96_{\pm 0.53}$ \\
& 100 & $\mathbf{40.31_{\pm 0.36}}$ & $38.42_{\pm 0.05}$ & $37.36_{\pm 0.51}$ & $37.94_{\pm 0.20}$ & $37.98_{\pm 0.18}$ & $38.42_{\pm 0.52}$ \\
\bottomrule
\end{tabular}}
\end{table}

\subsection{Calibration error}\label{sec:exp:ece}

A classifier is perfectly calibrated when $\mathbb{P}(\hat y=y\mid \hat p=p)=p$ for every confidence level $p$. The expected calibration error~\cite{ece} estimates the deviation by binning predictions into $M$ equal-width confidence intervals and averaging the gap between bin accuracy and bin confidence,
\begin{equation}
    \text{ECE} \;=\; \sum_{m=1}^M \tfrac{|B_m|}{N}\, \big| \text{conf}(B_m)-\text{acc}(B_m) \big|.
\end{equation}
All ECE values are computed from the raw softmax output of the trained classifier; we do not apply temperature scaling~\cite{oncalibration} or post-hoc rescaling. While effective in data-rich settings, these corrections need a validation split large enough to fit rescaling parameters. The small-data setting of this paper does not support such a split, and using the training set for this purpose would undermine our comparison.
GenCE attains the lowest ECE in five of six configurations of Table~\ref{tab:ece_mix}. The single exception is CIFAR-10 with $N{=}5000$, where focal loss, an objective engineered to suppress the high-confidence tail, edges ahead. Still, it pays more than a point of accuracy for that calibration. The advantage over CE is most pronounced on Mini-ImageNet, where GenCE substantially reduces the calibration gap at both sample sizes. The shape of the loss makes this almost a corollary of how it optimizes: pointwise objectives reward inflating $p_\theta(y_i|\x_i)$ on each example independently, whereas GenCE divides that score by the model's average confidence on the same class across the batch, so a uniform inflation is a no-op at the objective level. Calibration improves not because of an extra regularizer but because indiscriminate confidence is no longer rewarded.

\begin{table}[htbp]
\centering
\vspace{-10pt}
\caption{Expected Calibration Error (ECE, \%) across datasets and training-set sizes (ResNet). Lower is better. The best result in each row is in \textbf{bold}.}
\label{tab:ece_mix}
\setlength{\tabcolsep}{6pt}
\renewcommand{\arraystretch}{1.}
\resizebox{\linewidth}{!}{%
\begin{tabular}{l c *{6}{c}}
\toprule
\textbf{Dataset} & \textbf{$N$} & \textbf{GenCE} & \textbf{CE} & \textbf{GCE} & \textbf{Focal Loss} & \textbf{MAE} & \textbf{Brier Loss} \\
\midrule
\multirow{2}{*}{CIFAR-10}
  & 5000 & $6.82_{\pm 0.55}$ & $7.82_{\pm 0.67}$ & $9.61_{\pm 0.43}$ & $\mathbf{4.48_{\pm 0.33}}$ & $9.17_{\pm 0.19}$ & $7.65_{\pm 0.59}$ \\
  & 2000 & $\mathbf{10.64_{\pm 0.39}}$ & $14.82_{\pm 0.74}$ & $17.77_{\pm 1.24}$ & $11.05_{\pm 0.60}$ & $16.00_{\pm 0.55}$ & $14.38_{\pm 1.28}$ \\
\midrule
\multirow{2}{*}{CIFAR-100}
  & 5000 & $\mathbf{10.25_{\pm 0.69}}$ & $15.14_{\pm 0.44}$ & $19.72_{\pm 0.67}$ & $10.98_{\pm 0.66}$ & $17.65_{\pm 0.43}$ & $14.26_{\pm 0.77}$ \\
  & 2000 & $\mathbf{22.32_{\pm 0.80}}$ & $28.34_{\pm 0.74}$ & $35.31_{\pm 1.16}$ & $25.88_{\pm 0.29}$ & $31.18_{\pm 1.01}$ & $27.58_{\pm 0.41}$ \\
\midrule
\multirow{2}{*}{Mini-ImageNet}
  & 5000 & $\mathbf{13.24_{\pm 0.92}}$ & $23.00_{\pm 0.71}$ & $28.88_{\pm 0.26}$ & $18.23_{\pm 1.02}$ & $25.04_{\pm 0.61}$ & $22.07_{\pm 0.13}$ \\
  & 2000 & $\mathbf{23.82_{\pm 0.17}}$ & $32.94_{\pm 0.93}$ & $35.83_{\pm 0.73}$ & $30.39_{\pm 0.48}$ & $34.49_{\pm 0.87}$ & $32.55_{\pm 0.59}$ \\
\bottomrule
\end{tabular}}
\end{table}

Figure~\ref{fig:reliability} shows the same effect graphically. The pointwise baselines (CE, GCE, Focal Loss, MAE, Brier Loss) all sit below the diagonal in the high-confidence bins—the bins that dominate ECE—while GenCE hugs the diagonal across the full confidence range.

\begin{figure}[t]
    \centering
    \includegraphics[width=0.9\linewidth]{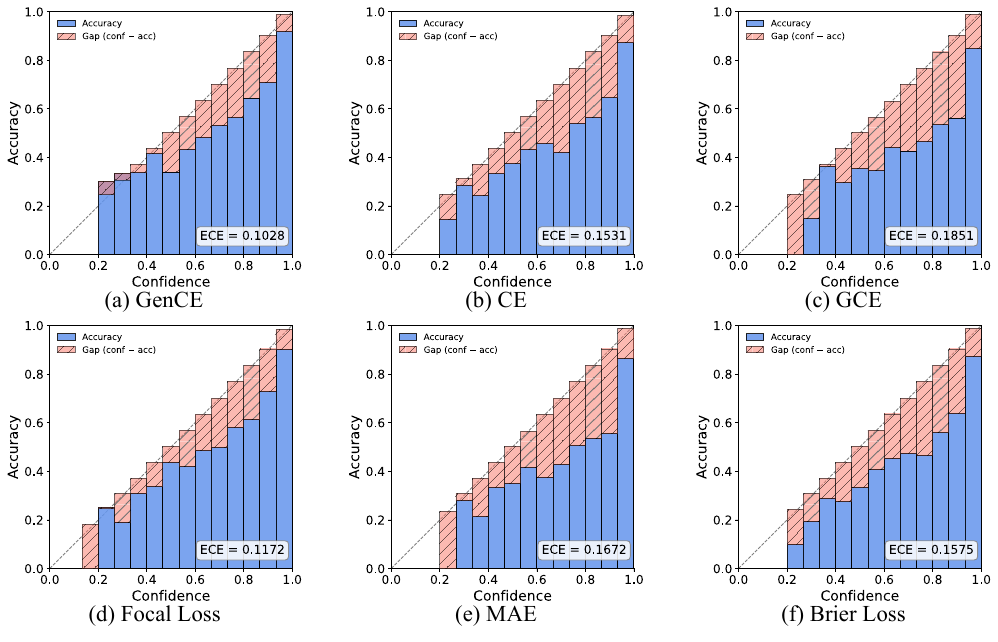}
    \vspace{-10pt}
    \caption{Reliability diagrams computed on CIFAR-10 with 15 bins, the dashed line ($y=x$) marks perfect calibration. Bars below (above) the diagonal indicate over- (under-) confidence; the gap between each bar and the diagonal contributes to ECE.}
    \label{fig:reliability}
    \vspace{-10pt}
\end{figure}

\subsection{Out-of-distribution detection}\label{sec:exp:ood}

Following~\cite{hendrycks2016baseline}, we score each input by softmax entropy ($\text{AUC}_\text{entropy}$) and by negative maximum softmax probability ($\text{AUC}_\text{confidence}$); AUROCs are reported in Table~\ref{tab:ood}. GenCE leads in every cell. The Near-OOD margin is modest (about a point), but the Far-OOD entropy gap is substantial, and equally telling is the variance: GenCE's seed-to-seed standard deviation is $1.33$, against CE's $3.43$. The CE failure mode is concrete: on individual seeds, it produces highly confident labels for SVHN digits that bear no resemblance to any CIFAR-10 class, collapsing the entropy gap on those runs. Batch-level normalization quietly suppresses this: a confident answer must be relatively confident relative to the rest of the batch, which is hard to satisfy for uniformly unfamiliar inputs.

\begin{table*}[htbp]
\vspace{-5pt}
  \centering
  \small
  \caption{AUROC (\%) of models trained on CIFAR-10 as in-distribution data and evaluated on Near-OOD (CIFAR-100) and Far-OOD (SVHN) regimes. Each entry reports mean$_{\pm\text{std}}$ over 5 random seeds. Higher is better. The best value in each row is in \textbf{bold}.}
  \vspace{-6pt}
  \label{tab:ood}
  \resizebox{\linewidth}{!}{%
  \begin{tabular}{llcccccc}
    \toprule
    \textbf{Dataset} & \textbf{Metric}
      & \textbf{GenCE}
      & \textbf{CE}
      & \textbf{GCE}
      & \textbf{Focal Loss}
      & \textbf{MAE}
      & \textbf{Brier Loss} \\
    \midrule
    \multirow{2}{*}{Near-OOD}
      & $\text{AUC}_\text{entropy}$
        & $\mathbf{87.22_{\pm 0.11}}$ & $86.24_{\pm 0.47}$ & $86.60_{\pm 0.62}$  & $86.36_{\pm 0.62}$
        & $86.18_{\pm 0.20}$ & $86.79_{\pm 0.17}$  \\
      & $\text{AUC}_\text{confidence}$
        & $\mathbf{86.98_{\pm 0.09}}$ & $86.06_{\pm 0.46}$ & $86.40_{\pm 0.60}$ & $86.14_{\pm 0.58}$
        & $86.03_{\pm 0.20}$ & $86.56_{\pm 0.17}$  \\
    \midrule
    \multirow{2}{*}{Far-OOD}
      & $\text{AUC}_\text{entropy}$
        & $\mathbf{94.87_{\pm 1.33}}$ & $88.40_{\pm 3.43}$ & $90.41_{\pm 2.09}$ & $87.57_{\pm 3.20}$
        & $91.61_{\pm 1.00}$ & $93.84_{\pm 0.50}$ \\
      & $\text{AUC}_\text{confidence}$
        & $\mathbf{94.33_{\pm 1.32}}$  & $88.16_{\pm 3.19}$ & $90.00_{\pm 2.04}$ & $87.31_{\pm 3.95}$
        & $91.34_{\pm 0.89}$ & $93.33_{\pm 0.48}$ \\
    \bottomrule
  \end{tabular}}
\end{table*}

\subsection{Sensitivity analysis}\label{sec:exp:sens}

\setlength{\columnsep}{15pt}
\begin{wrapfigure}[12]{r}{0.37\linewidth}
    \centering
    \vspace{-0pt}
    \includegraphics[width=1\linewidth]{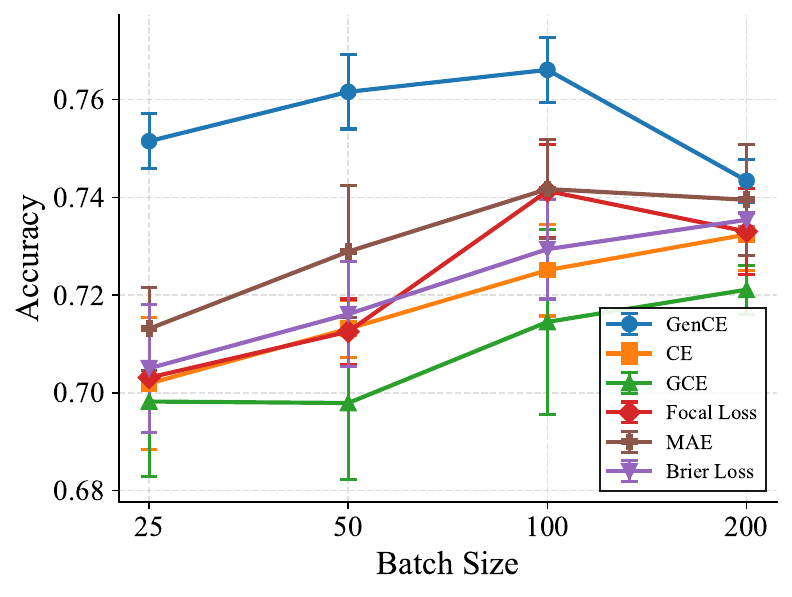}
    \vspace{-12pt}
    \caption{Accuracy vs.\ batch size $B$ for all losses on $N{=}2000$ CIFAR-10.}
    \label{fig:bs_cifar10}
    \vspace{-10pt}
\end{wrapfigure}

{\hyphenpenalty=10000 \exhyphenpenalty=10000 \emergencystretch=3em
\textbf{Effect of batch size.} Because GenCE ties each sample's score to the rest of its mini-batch, a natural concern is whether the loss becomes more fragile under changes in $B$ than its pointwise counterparts. Sweeping $B\in\{25,50,100,200\}$ on $N{=}2000$ CIFAR-10 with everything else held fixed, Figure~\ref{fig:bs_cifar10} indicates otherwise: GenCE traces essentially the same accuracy-versus-$B$ profile as CE and the other baselines, peaking in the middle of the range and degrading similarly at $B{=}25$, where per-batch averages become too noisy, and at $B{=}200$, where only ten batches per epoch leave the optimizer too few updates. The within-batch normalization, therefore, introduces no batch-size sensitivity beyond that already imposed by mini-batch training.
\par}
\section{Conclusion}
Generative Cross-Entropy is a small, theoretically motivated change to the cross-entropy objective: each sample's softmax score is divided by the model's average response to the same class across the batch. The modification adds no parameters, no separate density model, and no architectural change, yet it carries the sample-efficiency advantage of generative classifiers into a fully discriminative network. We extended the proper-scoring-rule framework to this non-local setting and showed that the population risk of GenCE is uniquely minimized at the true posterior under a mild completeness condition. Across CIFAR-10/100 and Mini-ImageNet, two backbones, and both balanced small-data and class-imbalanced regimes, GenCE-trained classifiers are more accurate, better calibrated, and more reliable on out-of-distribution inputs than CE, with the gap widening as data becomes scarcer or more imbalanced—evidence that the loss itself, more than the data, is a useful and under-examined axis for label-efficient learning.

\textbf{Limitations and future directions.}
The accuracy of GenCE depends on how faithfully each batch reflects the class marginal $p_\theta(y)$; the sensitivity analysis (Section~\ref{sec:exp:sens}) confirms robustness across practical batch sizes, with class-stratified sampling available for rare-class regimes. Two directions remain open: tightening the population-level guarantee of Theorem~\ref{thm:properness} into a finite-sample analysis, and applying GenCE to text, speech, and structured prediction together with the augmentation or pretraining pipelines it is designed to complement.

\bibliographystyle{plain}
\bibliography{neurips_2026}

\newpage
\appendix

\section{Technical Appendices and Supplementary Material}

\subsection{Proof of Theorem~\ref{thm:properness}}\label{proof:properness}
The proof of Theorem~\ref{thm:properness} relies on an expectation form of the log-sum inequality, which we state and prove first; the strict-properness argument then follows.

\begin{lemma}[Expectation form of log-sum inequality]
\label{lemma:log sum}
For any two bounded measurable nonnegative functions $h,g:\X\to\R_{\ge 0}$ and any distribution $\mathcal{P}_\X$,
\begin{equation}
\label{eq:exp-log-sum}
    \E_{\x\sim\mathcal{P}_\X}\!\left[g(\x)\log\frac{g(\x)}{h(\x)}\right]
    \;\ge\;
    \E_{\x\sim\mathcal{P}_\X}\!\left[g(\x)\right]\log
    \frac{\E_{\x\sim\mathcal{P}_\X}[g(\x)]}{\E_{\x\sim\mathcal{P}_\X}[h(\x)]},
\end{equation}
with equality if and only if $g(\x)=c\cdot h(\x)$ holds $\mathcal{P}_\X$-almost surely for some constant $c\ge 0$. Furthermore, if $h$ and $g$ are continuous, then equality holds if and only if $g(\x)=c\cdot h(\x)$ for every $\x\in\X$.
\end{lemma}
\begin{proof}
Throughout, write $\E[\cdot]$ for $\E_{\x\sim\mathcal{P}_\X}[\cdot]$ and adopt the standard conventions $0\log 0:=0$ and, for $a>0$, $a\log(a/0):=+\infty$.

\textbf{Reduction to the nondegenerate case.}
If $\E[h]=0$, then $h=0$ a.s. If furthermore $g=0$ a.s., both sides of~\eqref{eq:exp-log-sum} equal $0$ and equality holds with $c=0$; otherwise $\E[g]>0$ and both sides equal $+\infty$. If instead $\E[g]=0$, then $g=0$ a.s.\ and both sides equal $0$. We may therefore assume $\E[h]>0$ and $\E[g]>0$. Moreover, if $\mathcal{P}_\X\{h=0,\,g>0\}>0$, then $g\log(g/h)=+\infty$ on that set, so the left-hand side of~\eqref{eq:exp-log-sum} is $+\infty$ and the inequality is trivial. We may therefore further assume
\begin{equation}\label{eq:supp-cond}
    g(\x)=0 \quad \mathcal{P}_\X\text{-a.s.\ on }\{h=0\}.
\end{equation}

\textbf{An auxiliary probability measure.}
Define a probability measure $\mathcal{Q}_\X$ on $\X$ by
\[
    \frac{\text{d}\mathcal{Q}_\X}{\text{d}\mathcal{P}_\X}(\x) \;=\; \frac{h(\x)}{\E[h]}.
\]
This is well defined because $h\ge 0$ and $\E[h]>0$, and it satisfies $\mathcal{Q}_\X(\{h=0\})=0$. Define $T:\X\to[0,\infty]$ by
\[
    T(\x)\;=\;\begin{cases} g(\x)/h(\x), & h(\x)>0,\\ 0, & h(\x)=0;\end{cases}
\]
the value on $\{h=0\}$ is immaterial $\mathcal{Q}_\X$-a.s.

\textbf{Jensen's inequality.}
The function $\varphi(t)=t\log t$, with $\varphi(0):=0$, is continuous on $[0,\infty)$ and strictly convex on $(0,\infty)$. Jensen's inequality under $\mathcal{Q}_\X$ gives
\begin{equation}\label{eq:jensen}
    \E_{\mathcal{Q}_\X}[\varphi(T)]\;\ge\;\varphi(\E_{\mathcal{Q}_\X}[T]).
\end{equation}
Computing both sides under $\mathcal{P}_\X$ and using~\eqref{eq:supp-cond}:
\[
    \E_{\mathcal{Q}_\X}[T]
    \;=\;\frac{1}{\E[h]}\,\E[g\,\mathbf{1}_{\{h>0\}}]
    \;=\;\frac{\E[g]}{\E[h]},
\]
and, since $g\log(g/h)=0$ on $\{h=0\}$ by~\eqref{eq:supp-cond} and the convention $0\log 0=0$,
\[
    \E_{\mathcal{Q}_\X}[\varphi(T)]
    \;=\;\frac{1}{\E[h]}\,\E\!\left[g\log\frac{g}{h}\right].
\]
Substituting into~\eqref{eq:jensen} yields
\[
    \frac{\E[g\log(g/h)]}{\E[h]}\;\ge\;\frac{\E[g]}{\E[h]}\log\frac{\E[g]}{\E[h]}.
\]
Multiplying by $\E[h]>0$ gives~\eqref{eq:exp-log-sum}.

\textbf{Equality case.}
Strict convexity of $\varphi$ on $(0,\infty)$ implies that equality in~\eqref{eq:jensen} holds if and only if $T$ is $\mathcal{Q}_\X$-a.s.\ constant; that is, there exists $c\ge 0$ such that $g(\x)/h(\x)=c$ for $\mathcal{P}_\X$-a.e.\ $\x\in\{h>0\}$. Combined with~\eqref{eq:supp-cond}, this is equivalent to $g(\x)=c\cdot h(\x)$ $\mathcal{P}_\X$-a.s. The converse is immediate. If, in addition, $h$ and $g$ are continuous, then the set $A=\{\x\in\X:g(\x)=c\cdot h(\x)\}$ is closed (as the zero set of a continuous function). Since $\mathcal{P}_\X(A)=1$ and $\X=\text{supp}(\mathcal{P}_\X)$ is the smallest closed set of full measure, $\X\subseteq A\subseteq \X$, and $g(\x)=c\cdot h(\x)$ holds for every $\x\in\X$.
\end{proof}

We now turn to the proof of Theorem~\ref{thm:properness}. Recall from Section~\ref{sec:methods:problem} that the true posterior $q(\cdot|\cdot)$ and the classifier $p_\theta(\cdot|\cdot)$ are both continuous on $\X$; this lets us upgrade the equality case of Lemma~\ref{lemma:log sum} from ``$\mathcal{P}_\X$-a.s.''\ to ``everywhere on $\X$'' throughout the proof.

\begin{proof}[Proof of Theorem~\ref{thm:properness}]
The population risk of GenCE can be written as
\begin{equation}
    \begin{aligned}
        \mathcal{R}_\text{GenCE}(\mathcal{P}_\X, q, p_\theta)
        & = -\E_{\x\sim\mathcal{P}_\X}\,\E_{y\sim q(\cdot|\x)}\!\left[\log\frac{p_\theta(y|\x)}{\E_{\x'\sim\mathcal{P}_\X}[p_\theta(y|\x')]}\right] \\
        & = -\E_{\x\sim\mathcal{P}_\X}\sum_{k=1}^K q(k|\x)\log\frac{p_\theta(k|\x)}{\E_{\x'\sim\mathcal{P}_\X}[p_\theta(k|\x')]}.
    \end{aligned}
\end{equation}
Denote $\pi_k:=\E_{\x\sim\mathcal{P}_\X}[q(k|\x)]$ and $s_k:=\E_{\x\sim\mathcal{P}_\X}[p_\theta(k|\x)]$. The excess risk is
\begin{equation}
    \begin{aligned}
        \Delta\mathcal{R}\;
        &:=\;\mathcal{R}_\text{GenCE}(\mathcal{P}_\X,q,p_\theta)-\mathcal{R}_\text{GenCE}(\mathcal{P}_\X,q,q)\\
        &=\;\E_{\x\sim\mathcal{P}_\X}\sum_{k=1}^K q(k|\x)\log\frac{s_k\,q(k|\x)}{\pi_k\,p_\theta(k|\x)}\\
        &=\;\sum_{k=1}^K \E_{\x\sim\mathcal{P}_\X}\!\left[q(k|\x)\log\frac{q(k|\x)}{p_\theta(k|\x)}\right]
        - \sum_{k=1}^K \pi_k\log\frac{\pi_k}{s_k},
    \end{aligned}
\end{equation}
where we used $\E_{\x}[q(k|\x)]=\pi_k$ in the last step. By Lemma~\ref{lemma:log sum} applied with $g(\x)=q(k|\x)$ and $h(\x)=p_\theta(k|\x)$ — both continuous on $\X$ by assumption — we have, for each $k\in[K]$,
\begin{equation}
    \E_{\x\sim\mathcal{P}_\X}\!\left[q(k|\x)\log\frac{q(k|\x)}{p_\theta(k|\x)}\right]
    \;\ge\;\pi_k\log\frac{\pi_k}{s_k},
\end{equation}
with equality if and only if there exists a constant $c_k\ge 0$ such that $q(k|\x)=c_k\,p_\theta(k|\x)$ for every $\x\in\X$ (the continuous version of the equality case in Lemma~\ref{lemma:log sum}). Summing over $k$ shows $\Delta\mathcal{R}\ge 0$, with equality if and only if for every $k$ there exists $c_k\ge 0$ with $q(k|\x)=c_k\,p_\theta(k|\x)$ for every $\x\in\X$.

It remains to show that, in the equality case, $c_k=1$ for every $k$, so that $q=p_\theta$ on $\X$. First, $c_k>0$ for every $k$: if $c_k=0$ for some $k$, then $q(k|\x)=0$ for every $\x\in\X$, which contradicts the completeness of $q$ (taking $\boldsymbol{\alpha}$ to be the $k$-th canonical vector gives $\sum_{j}\alpha_j\,q(j|\x)=q(k|\x)=0$ for every $\x\in\X$ with $\boldsymbol{\alpha}\neq\mathbf{0}$). Hence $p_\theta(k|\x)=q(k|\x)/c_k$ for every $\x\in\X$. Combining $\sum_k p_\theta(k|\x)=1$ with $\sum_k q(k|\x)=1$,
\begin{equation}\label{eq:final-completeness}
    \sum_{k=1}^K \frac{1-c_k}{c_k}\,q(k|\x)
    \;=\;\sum_{k=1}^K\!\left(\frac{1}{c_k}-1\right)q(k|\x)
    \;=\;\sum_{k=1}^K p_\theta(k|\x)-\sum_{k=1}^K q(k|\x)
    \;=\;0
    \quad\text{for every }\x\in\X.
\end{equation}
Setting $\alpha_k:=(1-c_k)/c_k$, completeness of $q$ implies $\alpha_k=0$ for all $k$, i.e.\ $c_k=1$ for all $k$. Therefore $q(\cdot|\x)=p_\theta(\cdot|\x)$ for every $\x\in\X$, which establishes the strict properness of $\mathcal{L}_\text{GenCE}$.
\end{proof}

\subsection{Sufficient conditions for completeness}\label{appendix:completeness}

We give two sufficient conditions under which the kernel $q(\cdot|\cdot)$ satisfies the completeness condition of Definition~\ref{def:completeness}.

\subsubsection*{Example 1: anchor points}
\textbf{Assumption.} For each class $k\in[K]$ there exists $\x_k\in\X$ such that $q(k|\x_k)>1/2$.

\begin{proof}
Suppose $\sum_k\alpha_k\,q(k|\x)=0$ for all $\x\in\X$. Set $M:=\max_k|\alpha_k|$ and pick $k^\star\in\arg\max_k|\alpha_k|$. We show $M=0$.

Assume toward contradiction that $M>0$.

\smallskip
\textit{Case 1: $\alpha_{k^\star}=M$.} Evaluating the identity at $\x_{k^\star}$,
\[
0 \;=\; \alpha_{k^\star}\,q(k^\star|\x_{k^\star})+\sum_{k\neq k^\star}\alpha_k\,q(k|\x_{k^\star}).
\]
Since $\alpha_k\ge -M$ for every $k$ and $\sum_{k\neq k^\star}q(k|\x_{k^\star})=1-q(k^\star|\x_{k^\star})$,
\[
\sum_{k\neq k^\star}\alpha_k\,q(k|\x_{k^\star})\;\ge\;-M\bigl(1-q(k^\star|\x_{k^\star})\bigr).
\]
Therefore
\[
0\;\ge\;M\,q(k^\star|\x_{k^\star})-M\bigl(1-q(k^\star|\x_{k^\star})\bigr)\;=\;M\bigl(2q(k^\star|\x_{k^\star})-1\bigr)\;>\;0,
\]
because $q(k^\star|\x_{k^\star})>1/2$, a contradiction.

\smallskip
\textit{Case 2: $\alpha_{k^\star}=-M$.} By the symmetric bound $\alpha_k\le M$,
\[
0\;\le\;-M\,q(k^\star|\x_{k^\star})+M\bigl(1-q(k^\star|\x_{k^\star})\bigr)\;=\;-M\bigl(2q(k^\star|\x_{k^\star})-1\bigr)\;<\;0,
\]
again a contradiction. Hence $M=0$ and $\boldsymbol{\alpha}=\mathbf{0}$.
\end{proof}

\subsubsection*{Example 2: Gaussian mixture model}
\textbf{Assumption.} The class-conditional densities are $\mathcal{P}(\x|y=k)=\phi_k(\x)=\mathcal{N}(\x;\boldsymbol{\mu}_k,\boldsymbol{\Sigma}_k)$ with $\boldsymbol{\Sigma}_k\succ 0$, mixing weights $\pi_k>0$, and pairwise distinct components $(\boldsymbol{\mu}_k,\boldsymbol{\Sigma}_k)\neq(\boldsymbol{\mu}_j,\boldsymbol{\Sigma}_j)$ for $k\neq j$. Then by Bayes' rule,
\[
q(k|\x)=\mathcal{P}(y=k|\x)=\frac{\pi_k\,\phi_k(\x)}{\sum_j \pi_j\,\phi_j(\x)}.
\]

\begin{proof}
Suppose $\sum_k\alpha_k\,q(k|\x)=0$ for all $\x\in\R^d$. The denominator of $q(k|\x)$ is strictly positive, so multiplying through gives
\begin{equation}\label{eq:gmm-reduction}
\sum_{k=1}^K \beta_k\,\phi_k(\x)=0\qquad\forall\x\in\R^d,\qquad \beta_k:=\alpha_k\pi_k.
\end{equation}
Since $\pi_k>0$, it suffices to show that~\eqref{eq:gmm-reduction} implies $\boldsymbol{\beta}=\mathbf{0}$, i.e.\ that distinct Gaussian densities are linearly independent.

\smallskip
\textit{Step 1 (analytic continuation).} The Fourier transform of $\phi_k$ is $\widehat{\phi}_k(\boldsymbol{\xi})=\exp\bigl(i\boldsymbol{\mu}_k^\top\boldsymbol{\xi}-\tfrac12\boldsymbol{\xi}^\top\boldsymbol{\Sigma}_k\boldsymbol{\xi}\bigr)$. Applying $\mathcal{F}$ to~\eqref{eq:gmm-reduction},
\[
F(\boldsymbol{\xi}):=\sum_k\beta_k\exp\bigl(i\boldsymbol{\mu}_k^\top\boldsymbol{\xi}-\tfrac12\boldsymbol{\xi}^\top\boldsymbol{\Sigma}_k\boldsymbol{\xi}\bigr)=0\qquad\forall\boldsymbol{\xi}\in\R^d.
\]
Since $F$ is entire on $\mathbb{C}^d$, the identity theorem gives $F\equiv 0$ on $\mathbb{C}^d$. Substituting $\boldsymbol{\xi}=-i\boldsymbol{\eta}$,
\begin{equation}\label{eq:gmm-laplace}
\sum_k \beta_k \exp\bigl(\boldsymbol{\mu}_k^\top\boldsymbol{\eta}+\tfrac12\boldsymbol{\eta}^\top\boldsymbol{\Sigma}_k\boldsymbol{\eta}\bigr)=0\qquad\forall\boldsymbol{\eta}\in\R^d.
\end{equation}

\smallskip
\textit{Step 2 (separating direction).} Choose a unit vector $\mathbf{u}\in\R^d$ such that
\begin{itemize}
\item $\mathbf{u}^\top\boldsymbol{\Sigma}_k\mathbf{u}\neq\mathbf{u}^\top\boldsymbol{\Sigma}_j\mathbf{u}$ whenever $\boldsymbol{\Sigma}_k\neq\boldsymbol{\Sigma}_j$,
\item $\mathbf{u}^\top\boldsymbol{\mu}_k\neq\mathbf{u}^\top\boldsymbol{\mu}_j$ whenever $\boldsymbol{\Sigma}_k=\boldsymbol{\Sigma}_j$ (so necessarily $\boldsymbol{\mu}_k\neq\boldsymbol{\mu}_j$).
\end{itemize}
The set of $\mathbf{u}$ violating either condition is a finite union of proper algebraic subvarieties of $S^{d-1}$ (quadrics $\mathbf{u}^\top(\boldsymbol{\Sigma}_k-\boldsymbol{\Sigma}_j)\mathbf{u}=0$ and hyperplanes $\mathbf{u}^\top(\boldsymbol{\mu}_k-\boldsymbol{\mu}_j)=0$), hence has surface measure zero; such $\mathbf{u}$ exists.

Define $a_k:=\mathbf{u}^\top\boldsymbol{\mu}_k$ and $b_k:=\mathbf{u}^\top\boldsymbol{\Sigma}_k\mathbf{u}>0$. By construction the pairs $\{(b_k,a_k)\}_{k=1}^K$ are pairwise distinct in lexicographic order. Setting $\boldsymbol{\eta}=t\mathbf{u}$ in~\eqref{eq:gmm-laplace},
\begin{equation}\label{eq:gmm-1d}
h(t):=\sum_k\beta_k\exp\bigl(\tfrac12 b_k t^2+a_k t\bigr)=0\qquad\forall t\in\R.
\end{equation}

\smallskip
\textit{Step 3 (asymptotics).} Let $k^\star$ be the unique index that maximizes $(b_k,a_k)$ lexicographically. Dividing~\eqref{eq:gmm-1d} by $\exp(\tfrac12 b_{k^\star}t^2+a_{k^\star}t)$,
\[
0=\beta_{k^\star}+\sum_{k\neq k^\star}\beta_k\exp\bigl(\tfrac12(b_k-b_{k^\star})t^2+(a_k-a_{k^\star})t\bigr).
\]
For each $k\neq k^\star$, either $b_k<b_{k^\star}$ or $b_k=b_{k^\star}$ with $a_k<a_{k^\star}$; in either case the exponent tends to $-\infty$ as $t\to+\infty$. Therefore $\beta_{k^\star}=0$.

\smallskip
\textit{Step 4 (induction).} Removing component $k^\star$ leaves $K-1$ pairwise distinct Gaussians satisfying the same identity. Iterating (formally, induction on $K$, with base case $K=1$ trivial since a Gaussian is nowhere zero) yields $\beta_k=0$ for every $k$, hence $\alpha_k=0$ for every $k$.
\end{proof}

\begin{remark}
Example~1 is a diagonal-dominance argument: the anchor point $\x_k$ isolates class $k$ enough to dominate the contributions from other classes. Example~2 is the classical identifiability of finite Gaussian mixtures, recast as linear independence of the posterior kernel.
\end{remark}


\subsection{Datasets details}\label{appendix:datasets}

This subsection collects general background on the benchmarks used in Section~\ref{sec:experiments}; the precise splits and augmentation pipeline are specified in Section~\ref{sec:exp:setup}.

\textbf{CIFAR-10 and CIFAR-100~\cite{cifar10}.}
Two widely used image-classification benchmarks consist of low-resolution natural images. CIFAR-10 covers ten coarse object categories (e.g.\ \emph{airplane}, \emph{automobile}, \emph{bird}); CIFAR-100 partitions a similarly sized image collection into one hundred fine-grained classes grouped under twenty superclasses. Both datasets are class-balanced and have become a standard testbed for studying sample efficiency, calibration, and robustness because they are small enough to allow rapid iteration but rich enough to expose meaningful differences between methods.

\textbf{Mini-ImageNet~\cite{miniimagenet}.}
A $100$-class subset of ImageNet~\cite{imagenet}, originally introduced as a few-shot benchmark and now used more broadly as a small-scale ImageNet substitute. Compared with CIFAR, Mini-ImageNet has higher input resolution and visually richer images, which makes it a useful intermediate setting between CIFAR-style benchmarks and full ImageNet.

\textbf{Small-data subsets.}
The small-data regime ($N\in\{2000,5000\}$) is constructed by stratified subsampling of the original training partition: each class contributes the same number of samples, drawn under seed-controlled randomness so that different seeds yield different subsamples.

\textbf{Long-tailed CIFAR-10/100~\cite{cifar10lt}.}
A standard class-imbalanced variant of CIFAR-10/100 in which the per-class training count decays exponentially with the class index, controlled by an imbalance factor $\rho$. The exact schedule and values of $\rho$ used here are stated in Section~\ref{sec:exp:setup}.

\textbf{Out-of-distribution evaluation.}
The OOD study takes CIFAR-10 as the in-distribution (ID) data. For \emph{Near-OOD}, CIFAR-100 serves as the negative set: it shares low-level image statistics with CIFAR-10 but introduces a semantic shift through new categories. For \emph{Far-OOD}, we use SVHN~\cite{svhn}, a digit-recognition dataset captured from natural street-view photographs, whose images differ from those in CIFAR-10 in both content and style.

\subsection{Baseline losses}\label{appendix:baselines}

We compare GenCE against five widely used training losses. For each baseline, we adopt the formulation and recommended hyperparameter values given in the corresponding original paper. Throughout, $\boldsymbol{p}=p_\theta(\cdot|\x)\in\Delta^{K-1}$ denotes the predicted distribution, $y$ the integer label, and $p_y$ the predicted probability of the true class.

\begin{itemize}
    \item \textbf{Cross-entropy (CE):} $\mathcal{L}_\text{CE}=-\log p_y$.
    \item \textbf{Mean absolute error (MAE):} $\mathcal{L}_\text{MAE}=\sum_{k=1}^K |p_k-\mathbf{1}_{[k=y]}|$.
    \item \textbf{Brier loss~\cite{brierscore}:} $\mathcal{L}_\text{Brier}=\sum_{k=1}^K (p_k-\mathbf{1}_{[k=y]})^2$.
    \item \textbf{Generalised cross-entropy (GCE):} $\mathcal{L}_\text{GCE}=(1-p_y^q)/q$ for a hyperparameter $q\in(0,1]$.
    \item \textbf{Focal loss~\cite{lin2017focal}:} $\mathcal{L}_\text{Focal}=-(1-p_y)^\gamma\log p_y$ for a focusing parameter $\gamma\ge 0$.
\end{itemize}

GenCE and all baselines share the same architecture, optimizer, learning-rate schedule, augmentation pipeline, and batch composition, so any performance differences are attributable solely to the loss.

\subsection{Evaluation metrics}\label{appendix:eval}

\textbf{Top-1 accuracy.}
The fraction of test inputs whose top-scoring class matches the ground truth, evaluated under a deterministic forward pass with the model in evaluation mode.

\textbf{Expected Calibration Error~\cite{ece}.}
ECE quantifies the discrepancy between predicted confidence and empirical accuracy. Predictions are partitioned into $M$ equal-width bins on the maximum-confidence axis, and the per-bin gap between mean confidence and empirical accuracy is averaged with bin-size weights:
\begin{equation*}
    \text{ECE} \;=\; \sum_{m=1}^M \frac{|B_m|}{N}\,\bigl|\text{conf}(B_m)-\text{acc}(B_m)\bigr|,
\end{equation*}
where $N$ is the test-set size. ECE is a standard summary statistic for calibration, but is known to depend on bin count and binning scheme; we therefore fix the binning across all losses, datasets, and seeds for fair comparison. Reliability diagrams (Figure~\ref{fig:reliability}) provide a complementary view by plotting per-bin accuracy against confidence.

\textbf{AUROC for out-of-distribution detection.}
We treat OOD detection as a binary classification problem: each test input is assigned a scalar score derived from the trained ID classifier, and AUROC is computed with the ID test set as the positive class and the OOD test set as the negative class. AUROC is threshold-free and equals the probability that a uniformly chosen ID input is scored higher than a uniformly chosen OOD input. 
No temperature scaling or post-hoc calibration is applied before scoring; both detectors operate on the raw softmax output of the trained classifier.

\subsection{Extra Comparison}
\label{appendix:extraexp}


To position GenCE in a broader context, we briefly compare it with two well-known training paradigms—supervised contrastive learning and label smoothing—and report a head-to-head comparison in Table~\ref{tab:acc_method_by_data}.

\textbf{Supervised contrastive learning (SupCon)~\cite{khosla2020supervised}.}
SupCon is a representation-learning objective designed to pull together examples of the same class on the unit hypersphere. Each sample $i$ in a batch serves as an anchor; its $\ell_2$-normalized embedding $z_i$ is compared with every other in-batch embedding $z_a$ via cosine similarity, and the loss for $i$ takes the form 
\[-\frac{1}{|P(i)|}\sum_{p\in P(i)}\log\!\frac{\exp(z_i\cdot z_p/\tau)}{\sum_{a\in A(i)}\exp(z_i\cdot z_a/\tau)},\] 
where $P(i)$ collects all other samples of the same class as $i$ and $A(i)$ the rest of the batch. Training proceeds in two stages: the encoder is first optimized with this contrastive objective, then frozen while a separate linear head is fitted using cross-entropy. SupCon thus operates entirely in embedding space and explicitly distinguishes positives from negatives. GenCE differs in each of these dimensions. It operates directly on the classifier's softmax output rather than on $\ell_2$-normalized embeddings; the comparison is between samples as candidates for a given class rather than between embedding pairs; and training is single-stage, no projection head, no augmentation-pair construction, and no separately fitted classifier. Whereas SupCon is most naturally analyzed in the mutual-information framework, GenCE admits a proper-scoring-rule analysis (Theorem~\ref{thm:properness}) and yields a population-level uniqueness guarantee at the true posterior that contrastive losses do not provide.

\textbf{Label smoothing (LS)~\cite{labelsmoothing}.}
Label smoothing replaces the one-hot target $\mathbf{e}_y$ with $(1-\varepsilon)\,\mathbf{e}_y+\varepsilon/K\,\mathbf{1}$, distributing a small mass uniformly across non-target classes; cross-entropy is then evaluated against this softened target. The mechanism is target-side rather than loss-side: the loss remains pointwise, only the regression target changes, and the practical effect is to discourage the network from saturating $p_\theta(y\mid\x)$ at $1$. This intervention sits at a different place in the pipeline than GenCE, which leaves the targets one-hot and instead reshapes the score itself by referencing the rest of the batch. Because the two modifications are largely orthogonal, they compose without conflict, and we evaluate the combination in Table~\ref{tab:acc_method_by_data} alongside their individual variants. We use the standard smoothing factor $\varepsilon=0.1$ throughout.

\begin{table}[htbp]
\centering
\caption{Classification accuracy (\%) on balanced small-data subsets ($N{=}2000$, ResNet-50) across datasets. Each entry reports mean$_{\pm\text{std}}$ over 5 random seeds. Note that SupCon does not support soft targets and is therefore not evaluated under the Label Smoothing setting.}
\label{tab:acc_method_by_data}
\setlength{\tabcolsep}{6pt}
\renewcommand{\arraystretch}{1.}
\resizebox{\linewidth}{!}{%
\begin{tabular}{l *{5}{c}}
\toprule
\textbf{Dataset} & \textbf{GenCE} & \textbf{GenCE w/ Label Smoothing} & \textbf{CE} & \textbf{CE w/ Label Smoothing} & \textbf{SupCon} \\
\midrule
CIFAR-10      & $\mathbf{76.61_{\pm 0.68}}$ & $76.48_{\pm 0.50}$ & $72.51_{\pm 0.94}$ & $73.14_{\pm 0.15}$ & $69.48_{\pm 0.43}$ \\
CIFAR-100     & $32.06_{\pm 0.47}$ & $\mathbf{32.60_{\pm 0.50}}$ & $31.15_{\pm 1.11}$ & $30.80_{\pm 1.04}$ & $23.44_{\pm 0.89}$ \\
Mini-ImageNet & $27.75_{\pm 0.44}$ & $\mathbf{27.98_{\pm 0.51}}$ & $24.91_{\pm 0.38}$ & $26.97_{\pm 0.71}$ & $20.50_{\pm 0.43}$ \\
\bottomrule
\end{tabular}
}
\end{table}

Table~\ref{tab:acc_method_by_data} surfaces two clear patterns. SupCon lags noticeably behind all other methods on all three datasets, consistent with the well-documented difficulty of training a contrastive encoder when per-class sample density is low, precisely the regime that motivates this paper. The effect of label smoothing on CE, by contrast, is dataset-dependent: it helps on some benchmarks, is roughly neutral on others, and slightly hurts elsewhere, in line with prior reports that the gain from smoothed targets is sensitive to dataset and architecture. The same intervention, applied on top of GenCE, has only a marginal effect overall. Across all configurations, the best results are achieved by GenCE, with or without smoothing, and the gap to both CE+LS and SupCon is preserved.


\end{document}